\documentclass[conference]{IEEEtran}
\usepackage{cite}
\usepackage{amsmath,amssymb,amsfonts,amsthm}
\usepackage{graphicx}
\usepackage{booktabs}
\usepackage{multirow}
\usepackage{url}

\newtheorem{proposition}{Proposition}
\newtheorem{corollary}{Corollary}

\makeatletter
\def\thm@space@setup{\thm@preskip=2pt plus 1pt minus 1pt\thm@postskip=2pt plus 1pt minus 1pt}
\makeatother

\begin{document}

\title{Neither Silence nor Overlap Is Failure: Intent-Conditioned Evaluation of Turn-Taking in Full-Duplex Spoken Dialogue Models}

\author{\IEEEauthorblockN{Kian Shamsaie, Iman Modarressi}
\IEEEauthorblockA{People Make Things\\
\textit{\{k,i\}@peoplemakethings.com}}}
\maketitle

\begin{abstract}
Benchmarks for full-duplex spoken dialogue models score turn-taking with binary fixed-window rules that reward immediate response or silence by completeness of the prior turn. We argue that the appropriateness of a response offset, whether delayed silence or anticipatory overlap, is conditional on the speaker's latent intent, identifiable only from that speaker's behavior. We introduce TACT, a benchmark of 9,728 episodes and 73.2 hours from five dyadic corpora; each episode carries dialogue history, a per-speaker memory profile, and an annotator-derived posterior over six intent classes. Scoring replaces binary windows with a strictly proper threshold-weighted continuous ranked probability score whose weights are intent-conditioned timing kernels fitted to human floor-transfer-offset distributions, proving boundedness, consistency, and binary reduction. Across eleven systems the best model reaches 0.47 against a human topline of 0.86, is nearly invariant to speaker profiles, and TACT agrees with human judgments at Spearman 0.81 versus 0.46 for binary metrics.
\end{abstract}


\section{Introduction}
\label{sec:intro}

Full-duplex spoken dialogue models (SDMs) listen and speak simultaneously, deciding continuously when to take the floor, yield it, backchannel, or remain silent \cite{nguyen2023generative,defossez2024moshi,veluri2024beyond,lin2022duplex}. Evaluation of these decisions has relied on binary event detection inside fixed look-ahead windows: Full-Duplex-Bench scores whether a model takes over within a window after a pause or turn end and how fast \cite{lin2025full}, its successor extending the same logic to overlap \cite{lin2026full}. Under such metrics, a response inside the window is correct, silence after completion is a miss, and incoming speech before completion is an intrusion, whatever the interlocutor was doing with that silence or inviting with that incompleteness.

This paper begins from a two-sided observation well established in conversation analysis but absent from machine evaluation: silence after a completed turn is not intrinsically a failure, and overlapping talk before completion is not intrinsically a violation. Human floor transfer offsets concentrate around 200~ms \cite{stivers2009universals,levinson2015timing,heldner2010pauses}, yet long gaps are routinely tolerated when the prior turn was rhetorical, floor-holding, or abandoned, or when the recipient is expected to reflect \cite{sacks1974simplest,schegloff2000overlapping,kendrick2015timing}; symmetrically, backchannels, collaborative completions, and urgent clarifications are routinely launched in overlap, before any turn end exists \cite{yngve1970getting,ward2000prosodic}. The value of a response offset therefore depends on the speaker's latent intent: the same trailing-off question followed by 1.5~s of silence can demand an immediate answer from one speaker and forbid one from another, and the same mid-clause incoming can be cooperative or disruptive. Crucially, intent is often identifiable only from that speaker's behavior across interactions, because speakers differ systematically in pause tolerance, backchannel solicitation, and floor-holding style \cite{reece2023candor,gravano2011turn}; an evaluation blind to intent and memory rewards the degenerate policy of answering everything quickly and never coming in early, on which current SDMs have converged.

We operationalize this observation in \emph{TACT} (Theory-of-mind Aware Conversational Timing), which supersedes fixed-window evaluation in three ways. First, every episode carries long multi-turn history, a per-speaker memory profile assembled from other episodes of the same speaker, and a six-class latent intent $z$ annotated by at least three humans and fused with a calibrated large-language-model judge into a posterior $\hat q(z \mid e)$. Second, scoring is continuous, two-sided, and strictly proper: a threshold-weighted continuous ranked probability score whose nonnegative weight functions $w_z(t)$ are derived from intent-conditioned asymmetric timing kernels $u_z(t)$ over the offset domain $\mathcal{T}$, so that negative offsets are anticipatory onsets overlapping the ongoing turn, fitted to human floor-transfer-offset distributions from CANDOR \cite{reece2023candor}, SSSD \cite{sheikh2025sssd}, and otoSpeech \cite{otospeech2025full} dyads, with analogous proper scores for length and prosody-weighted overlap; we prove boundedness, strict propriety with the human conditional law as unique maximizer, reduction of Full-Duplex-Bench to a degenerate weight limit, and identifiability of the intent mixture. Third, the benchmark is realistic in scale and provenance: 9,728 episodes and 73.2 hours from five ecologically diverse public dyadic corpora, plus synthetic probes for rare conditions.

Evaluating eleven frontier systems, we find that models appearing adequate under binary metrics collapse under intent conditioning. The best attains a composite of 0.47 against a human topline of 0.86, the ranking reorders substantially relative to a reproduced Full-Duplex-Bench composite (Spearman rank correlation 0.55), and all systems exhibit an \emph{over-eagerness} pathology whose rigidity is two-sided: near-uniform fast responding regardless of intent, paired with an inability to launch cooperative early onsets where humans routinely overlap. A memory-swap ablation shows that model behavior is essentially invariant to the speaker profile while held-out human behavior is not, and the TACT composite correlates with held-out human judgments at Spearman $\rho=0.81$ versus $0.46$ for the binary metrics it replaces.

\section{Related Work}
\label{sec:related}

Turn allocation was formalized by Sacks, Schegloff, and Jefferson \cite{sacks1974simplest}; smooth transfers with modal gaps near 200~ms are a cross-linguistic universal \cite{stivers2009universals} requiring predictive planning of turn ends \cite{levinson2015timing,deruiter2006projecting} from lexico-syntactic and intonational completion cues \cite{bogels2015listeners,gravano2011turn}. The floor-transfer-offset (FTO) distribution, long studied on telephone corpora \cite{godfrey1992switchboard}, is heavy-tailed and well fitted by ex-Gaussian forms tracking sequence organization and preference \cite{roberts2015effects,kendrick2015timing,heldner2010pauses}; reflective responses are delayed, rhetorical or floor-holding turns license long silence \cite{schegloff2000overlapping}, and backchannels occupy an overlap-permissive regime cued by prosody \cite{yngve1970getting,ward2000prosodic}: timing is a conditional distribution given intent, the structure our scoring adopts, with a taxonomy motivated by \cite{stolcke2000dialogue,clark1996using}. Computational treatments evolved from finite-state floor control \cite{raux2009finite,schlangen2011general} to self-supervised prediction: TurnGPT predicts turn shifts from lexical context \cite{ekstedt2020turngpt}, Voice Activity Projection (VAP) predicts joint future voice activity from stereo audio \cite{ekstedt2022voice} with real-time and multilingual extensions \cite{inoue2024realtime,inoue2024multilingual}, and production endpointers train jointly with recognition \cite{chang2022turn,skantze2021turn}; these predictors supply machinery we repurpose for scoring.

Full-duplex behavior appeared commercially in XiaoIce and task-oriented stacks \cite{zhou2020design,lin2022duplex}; dGSLM introduced end-to-end two-channel modeling \cite{nguyen2023generative}, followed by Moshi's parallel speech-text streams \cite{defossez2024moshi}, time-multiplexing in SyncLLM \cite{veluri2024beyond}, duplex fine-tuning of text LLMs \cite{wang2025freeze}, listen-while-speaking objectives \cite{ma2025language}, a family of speech-native assistants \cite{zhang2023speechgpt,xu2025qwen,openbmb2025minicpmo}, persona-controllable models such as PersonaPlex \cite{roy2026personaplex}, and proprietary realtime stacks from GPT-4o Realtime to gpt-realtime-2 and Gemini 3.1 Flash Live \cite{openai2024gpt4o,openai2026gptrealtime2,google2026geminilive}. Benchmarks for such systems largely score semantic quality \cite{chen2024voicebench,ao2024sdeval}. Interaction-level evaluation began with Full-Duplex-Bench, which defined the four categories we retain via takeover rates and latencies inside fixed windows \cite{lin2025full}, and has since grown well beyond fixed local windows: v1.5 probes four overlap scenarios with prosodic-adaptation metrics \cite{lin2026full}, FD-Bench scaled simulated interruptions \cite{peng2025fdbench}, v2 adds a multi-turn automated examiner \cite{lin2026fdbv2}, v3 uses entirely real human audio annotated for five disfluency categories under multi-step tool use \cite{lin2026fdbv3}, and Talking Turns trains a supervised judge on human corpora \cite{arora2025talking}. These increasingly cover overlap, multi-turn structure, and real disfluent audio, yet each defines the appropriate behavior as a function of observed acoustic and dialogue context rather than the speaker's latent intent. TACT does not claim to be the first multi-turn or overlap-aware benchmark; its contribution is to condition the target on a latent, speaker-specific, theory-of-mind-and-memory intent posterior and score continuous timing with a strictly proper, asymmetry-aware rule \cite{gneiting2011comparing,gneiting2007strictly} rather than binary windows, calibrating the judge following \cite{zheng2023judging,liu2023geval}. Finally, attributing latent mental states is the classical theory-of-mind task \cite{premack1978does}, at which text probes show large models brittle \cite{sap2022neural,kim2023fantom}; persona and memory help text dialogue \cite{zhang2018personalizing,xu2022beyond}, yet no spoken-dialogue evaluation tests whether a system adapts its timing to a specific speaker.

\section{The TACT Benchmark}
\label{sec:bench}

\subsection{Episode Construction}
A TACT episode $e$ consists of a dual-channel audio context $A_e$ ending at a decision point, a time-aligned transcript history $H_e$ averaging 3.2 minutes and 11.4 turns, a behavioral memory profile $M_{s(e)}$ of the focal speaker $s(e)$, and annotations; episodes and profiles are mined from five public dyadic corpora and ratings from at least three annotators fuse with a calibrated LLM judge into the intent posterior conditioning the scoring. The decision point is the end of a focal-speaker turn (for pause handling, an intra-turn silence of at least 600~ms; for interruption, a region where the system holds the floor). The system receives the focal-speaker channel as streaming input with a rendered profile; its output is recorded over $[-T_a,T_h]$ with pre-horizon $T_a=2$~s (so anticipatory overlap onsets are observable) and horizon $T_h=5$~s, and each episode is run three times, keeping the median-scoring run.

Episodes are mined from five public corpora chosen for ecological diversity. CANDOR (1,656 dyadic video calls with demographics and surveys \cite{reece2023candor}) anchors the fairness audit (2,560 episodes, 19.8~h); SSSD (727 hours of crowdsourced spontaneous English dyads \cite{sheikh2025sssd}) supplies scale and speaker diversity (2,304 episodes, 17.1~h); otoSpeech-full-duplex-280h, raw 48~kHz stereo FLAC with one channel per speaker and a cleaned 141-hour companion we use \cite{otospeech2025full,otospeech2025processed}, preserves the overlap ground truth mixed-channel corpora destroy, essential for fitting anticipatory laws (1,792 episodes, 13.9~h); Seamless Interaction \cite{agrawal2025seamless} seeds interruption episodes (1,536 episodes, 11.5~h); and the AMI individual-headset subset, with characterized ASR baselines (Whisper near 16--17\% WER on AMI-IHM \cite{radford2023robust}), supplies the sole multiparty contrast mined as dyadic exchanges (768 episodes, 5.4~h) \cite{carletta2005ami}. Synthetic TTS probes (768 episodes, 5.5~h) cover conditions sparse in found data. In total TACT comprises 9,728 episodes and 73.2 hours over the four categories of \cite{lin2025full} (2,720 pause handling, 2,432 backchanneling, 2,592 smooth turn-taking, 1,984 user interruption), each crossed with the intent taxonomy.

\subsection{Speaker Memory Profiles}
For every focal speaker we assemble a profile $M_s$ from episodes of the same speaker disjoint from the test episode (mean 14.6 minutes of evidence), summarizing the FTO distribution, intra-turn pause statistics, abstention tolerance, backchannel solicitation and production rates, articulation rate, and floor-holding devices, rendered as structured features and a prompt synopsis. Profiles are the only channel through which speaker-specific expectations can flow; the swap ablation of Section~\ref{sec:results} exploits this.

\subsection{Intent Taxonomy and Annotation Protocol}
Each decision point is annotated with a latent intent $z$ over $\mathcal{Z}=\{$\textsc{ans}, \textsc{ref}, \textsc{rhe}, \textsc{hold}, \textsc{bck}, \textsc{abn}$\}$: immediate-answer-seeking, reflective, rhetorical, floor-holding, cooperative-incoming-inviting, and turn-abandonment, distilled from dialogue-act inventories \cite{stolcke2000dialogue} and the timing literature \cite{kendrick2015timing,schegloff2000overlapping} so that classes induce distinct human timing on both sides of the turn end; \textsc{bck} covers continuers, collaborative completions, and urgent clarifications, whose normative onset lies in overlap. Three to five annotators (median three), shown full context, profile, and human continuation, label each episode, rate silence acceptability, and mark acceptable onset windows on both sides. Agreement is Krippendorff $\alpha=0.73$ for intent (nominal), $0.84$ for onset windows, and $0.79$ for overlap acceptability \cite{krippendorff2004content}; the empirical intent distribution is 31.6\% \textsc{ans}, 18.2\% \textsc{hold}, 16.1\% \textsc{ref}, 15.3\% \textsc{bck}, 9.6\% \textsc{abn}, 9.2\% \textsc{rhe}. A temperature-scaled LLM judge is fused in (Section~\ref{sec:scoring}); calibration cuts its expected calibration error from 8.4\% to 2.1\%, and the fusion predicts a held-out annotator's label better than either source alone.

Because timing norms vary across populations \cite{stivers2009universals}, TACT is stratified by gender, age band, first-language region, signal-to-noise ratio, and articulation-rate tercile from CANDOR and Seamless metadata \cite{reece2023candor,agrawal2025seamless}, and kernels are fitted per slice where data permit, so the metric does not encode one population's norm as universal.

\section{Intent-Conditioned Continuous Scoring}
\label{sec:scoring}

\subsection{Notation and Intent Posterior}
Let $e$ denote an episode with decision point at time origin $t=0$, focal speaker $s(e)$, history $H_e$, and profile $M_{s(e)}$. The evaluated system's action is $y_e \in \mathcal{T} \cup \{\varnothing\}$ on the bounded offset domain $\mathcal{T}=[-T_a,T_h]$, with pre-horizon $T_a=2$~s and evaluation horizon $T_h=5$~s: $y_e$ is the onset offset $t$ of its first responsive vocalization within $\mathcal{T}$ (negative offsets being anticipatory onsets in overlap with the ongoing turn), or the no-response-within-horizon outcome $\varnothing$ (the atom we score explicitly). With $K_e$ annotator labels $z^{(1)},\dots,z^{(K_e)}$, the empirical posterior is $\hat q_{\mathrm{A}}(z \mid e) = K_e^{-1}\sum_{k} \mathbf{1}[z^{(k)} = z]$, and the calibrated posterior fuses it with an LLM judge $p_{\mathrm{J}}$ \cite{zheng2023judging,liu2023geval}, temperature-scaled by $T^{\star}$ fitted on the development split,
\begin{equation}
\hat q(z \mid e) \;=\; (1-\lambda)\,\hat q_{\mathrm{A}}(z \mid e) \;+\; \lambda\,
\frac{p_{\mathrm{J}}(z \mid e)^{1/T^{\star}}}{\sum_{z'} p_{\mathrm{J}}(z' \mid e)^{1/T^{\star}}},
\label{eq:post}
\end{equation}
with $\lambda=0.3$ chosen by held-out annotator log-likelihood. Intuitively, \eqref{eq:post} forms a Bayesian consensus belief over speaker intent by blending majority human annotations with a calibrated judge. The posterior is a theory-of-mind estimate of the speaker's state \cite{premack1978does,sap2022neural}: $H_e$ alone leaves intent ambiguous in 38\% of episodes, and conditioning on $M_{s(e)}$ resolves over half of these.

\subsection{Reference Laws and Weights for All Six Intents}
\label{sec:laws}
For each of the six intents $z\in\mathcal{Z}$ we specify a human reference law $h_z=(1-\beta_z)\,f_z+\beta_z\,\delta_{\varnothing}$ on $\mathcal{T}=[-T_a,T_h]$ with the no-response atom $\varnothing$. Here, $\beta_z\in[0,1]$ is the empirical human no-response probability, estimated directly as the mean annotator silence acceptability on training dyads, while $f_z$ is the conditional continuous FTO density fitted on responsive turns disjoint from test. Because $\beta_z$ and $f_z$ capture the discrete no-response atom and response-conditional timing respectively, they represent distinct, non-conflicting components of $h_z$, with inter-annotator label disagreements retained in the posterior $\hat q(z\mid e)$. Following \cite{roberts2015effects}, the responsive intents \textsc{ans}, \textsc{bck}, and \textsc{abn} use the ex-Gaussian density with parameters $(\mu_z,\sigma_z,\tau_z)$ equal to $(0.16,0.11,0.24)$, $(-0.32,0.18,0.20)$, and $(1.95,0.50,0.80)$~s; the \textsc{bck} mode of $-0.19$~s places $72\%$ of onset mass inside the ongoing turn. The reflective intent \textsc{ref} uses a shifted log-normal of median $1.9$~s. Because the optimal action under a rhetorical (\textsc{rhe}) or floor-holding (\textsc{hold}) turn is to withhold the floor, $f_{\textsc{rhe}}$ and $f_{\textsc{hold}}$ are heavy-tailed shifted log-normals of median $2.6$~s and $3.0$~s, so the rare licensed response is late and most reference mass sits on $\varnothing$. The fitted no-response masses, equal to the mean annotator acceptability of silence, are $\beta_{\textsc{ans}}=0.05$, $\beta_{\textsc{ref}}=0.40$, $\beta_{\textsc{rhe}}=0.85$, $\beta_{\textsc{hold}}=0.90$, $\beta_{\textsc{bck}}=0.45$, $\beta_{\textsc{abn}}=0.50$.

The asymmetric, intent-conditioned importance of different offset regions is carried by a nonnegative weight $w_z$ derived from the same kernel shape,
\begin{equation}
u_z(t) = \Big(\tfrac{f_z(t)}{\sup_{t'} f_z(t')}\Big)^{\!\gamma},
\quad
w_z(t) = (1-\epsilon)\,u_z(t) + \epsilon,
\label{eq:kernel}
\end{equation}
with $\gamma=0.7$ and $\epsilon=0.02$, so $w_z>0$ on $\mathcal{T}$ (guaranteeing strict propriety below) while preserving the kernel's asymmetry up to the floor $\epsilon$: for \textsc{ans} late offsets are weighted above premature overlap, while for \textsc{bck} a well-placed overlapping onset is weighted above any post-completion one (Fig.~\ref{fig:kernels}).

\subsection{Threshold-Weighted Continuous Ranked Probability Scoring}
\label{sec:twcrps}
A model on episodes with posterior mode $z$ emits a predictive law $P_z$ over $t\in\mathcal{T}$ with a no-response atom, summarized by its sub-probability CDF $P_z(t)=\Pr(y\le t)$, whose defect $1-P_z(T_h^-)=1-\lim_{t\uparrow T_h}P_z(t)$ represents the predicted no-response mass on $\varnothing$ at horizon $T_h$. We score $P_z$ against a realized continuation $y$ by the threshold-weighted continuous ranked probability score (twCRPS) of \cite{gneiting2011comparing}, built on the CRPS of \cite{matheson1976scoring,gneiting2007strictly},
\begin{equation}
\mathrm{twCRPS}(P_z,y) \;=\; \int_{\mathcal{T}} w_z(t)\,\big(P_z(t)-\mathbf{1}[\,y\le t\,]\big)^2\,dt,
\label{eq:twcrps}
\end{equation}
where for the atom outcome $y=\varnothing$ the indicator $\mathbf{1}[\,y\le t\,]=0$ for every finite $t\in\mathcal{T}$, so withheld floor and post-completion onset are scored by one object. Intuitively, twCRPS measures the integrated squared discrepancy between the model's predictive cumulative distribution and the empirical step function $\mathbf{1}[y\le t]$ of the realized event, weighted by $w_z(t)$ so that timing errors are penalized strictly in proportion to their communicative severity under intent $z$. The negatively oriented twCRPS induces the proper divergence $d_z(P_z\,\|\,h_z)=\mathbb{E}_{y\sim h_z}[\mathrm{twCRPS}(P_z,y)]-\mathbb{E}_{y\sim h_z}[\mathrm{twCRPS}(h_z,y)]$, and the population timing metric maps it to $[0,1]$ as
\begin{equation}
S_{\mathrm{T}} \;=\; \sum_{z\in\mathcal{Z}} \pi(z)\left(1-\frac{d_z(P_z\,\|\,h_z)}{W_z}\right),
\qquad W_z=\!\int_{\mathcal{T}}\! w_z(t)\,dt,
\label{eq:stime}
\end{equation}
with $\pi(z)=\Pr(z)$ the intent prior. The per-episode score uses the single-sample twCRPS against the posterior-mixed reference $P_{\hat q}=\sum_z\hat q(z\mid e)\,h_z$, namely $s_{\mathrm{T}}(e)=1-\mathrm{twCRPS}(P_{\hat q},y_e)/\overline{W}_e$ with $\overline{W}_e=\sum_z\hat q(z\mid e)W_z$; its episode mean estimates $S_{\mathrm{T}}$ up to the irreducible single-continuation variance reconciled below.

\begin{figure}[t]
\centering
\includegraphics[width=0.88\columnwidth]{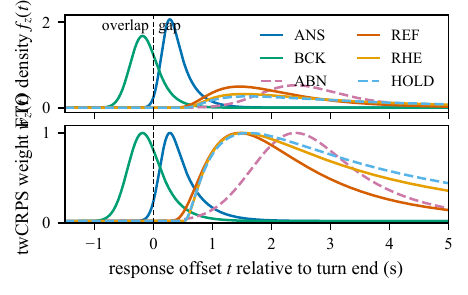}
\caption{Per-intent human FTO densities $f_z$ (top) and induced twCRPS weights $w_z(t)$ of \eqref{eq:kernel} (bottom) on $\mathcal{T}$: mass left of the turn end (dashed) is anticipatory overlap, the \textsc{bck} mode lies inside the ongoing turn, and \textsc{rhe}/\textsc{hold} place most reference mass on the no-response atom $\beta_z$ (dotted) with a heavy late tail.}
\label{fig:kernels}
\end{figure}

\subsection{Length, Overlap, and Composite}
Length is scored by the same proper construction: with $g_z$ the human duration density given intent and context bucket and $Q_z$ the model's predictive duration CDF, the aggregate is the bounded complement of the duration twCRPS divergence with the uniform weight $w\equiv1$ on the duration support,
\begin{equation}
S_{\mathrm{L}} \;=\; \sum_{z} \pi(z)\left(1 - \frac{d_z^{\mathrm{L}}(Q_z\,\|\,g_z)}{W^{\mathrm{L}}}\right),
\label{eq:length}
\end{equation}
where $d_z^{\mathrm{L}}$ is the CRPS divergence of \cite{matheson1976scoring} and $W^{\mathrm{L}}$ the length of the duration support, so $S_{\mathrm{L}}\in[0,1]$ and equals $1$ uniquely at $Q_z=g_z$. Overlap is scored with prosodic-completion weighting: a VAP-style estimator \cite{ekstedt2022voice,inoue2024realtime}, calibrated on annotator completion marks, assigns each instant $\tau$ of focal-speaker speech a completion probability $\psi(\tau)\in[0,1]$, and competitive (non-backchannel) overlap onsets $o$ with durations $d_o$ incur
\begin{equation}
s_{\mathrm{O}}(e) \;=\; \exp\!\Big(-\kappa \sum_{o\in O_e} \big(1-\psi(\tau_o)\big)\, d_o\Big),
\qquad \kappa = 1.2~\mathrm{s}^{-1},
\label{eq:overlap}
\end{equation}
so overlap at points of high prosodic completion, where competitive incoming is licensed \cite{bogels2015listeners,gravano2011turn}, is barely penalized while mid-clause interruption is penalized in proportion to duration; backchannels are exempt \cite{lin2025full}. Paired with the \textsc{bck} reference law's anticipatory mass, cooperative overlap is a high-scoring action while disruptive interruption remains costly: neither silence nor overlap is a category, only an action under a posterior. Each category aggregates its components as a fixed convex combination (timing dominates pause handling and smooth turn-taking, overlap dominates interruption, timing and length share backchanneling); the composite $S$ is the unweighted mean of the four category scores, with BCa bootstrap intervals over $10^4$ resamples \cite{efron1993introduction}.

\subsection{Formal Guarantees}

\begin{proposition}[Boundedness and strict propriety]
\label{prop:bound}
Fix an intent $z$ with $\pi(z)>0$ and weight $w_z$ strictly positive on the interior of $\mathcal{T}$ and integrable, with $W_z=\int_{\mathcal{T}}w_z<\infty$. For every predictive law $P_z$ on $\mathcal{T}\cup\{\varnothing\}$ and every realized $y\in\mathcal{T}\cup\{\varnothing\}$, $\mathrm{twCRPS}(P_z,y)\in[0,W_z]$, so the per-intent term and hence $S_{\mathrm{T}}$ of \eqref{eq:stime} lie in $[0,1]$. Moreover $\mathbb{E}_{y\sim h_z}[\mathrm{twCRPS}(P_z,y)]$ is minimized over all $P_z$ uniquely at $P_z=h_z$ (the twCRPS is strictly proper), so $S_{\mathrm{T}}=1$ if and only if $P_z=h_z$ for all such $z$; in particular every deterministic forecast and every law supported on the post-completion half-line incurs $S_{\mathrm{T}}<1$ whenever $h_z$ is dispersed or carries anticipatory mass.
\end{proposition}
\begin{proof}
Boundedness: for each $t$, $\big(P_z(t)-\mathbf{1}[y\le t]\big)^2\in[0,1]$ since $P_z(t)\in[0,1]$ and the indicator is in $\{0,1\}$; multiplying by $w_z(t)\ge0$ and integrating gives $\mathrm{twCRPS}\in[0,W_z]$, so $1-\mathrm{twCRPS}/W_z\in[0,1]$ and the convex combination \eqref{eq:stime} stays in $[0,1]$.

Strict propriety: write $H_z(t)=\Pr_{y\sim h_z}(y\le t)$ for the proper CDF of the continuous part on $\mathcal{T}$; the atom $\varnothing$ contributes $0$ to $\mathbf{1}[y\le t]$ at every finite $t$, hence $\mathbb{E}_{y\sim h_z}\mathbf{1}[y\le t]=H_z(t)$ for all $t\in\mathcal{T}$. For fixed $t$ the indicator is Bernoulli with mean $H_z(t)$, so the bias--variance identity gives $\mathbb{E}\big(P_z(t)-\mathbf{1}[y\le t]\big)^2=\big(P_z(t)-H_z(t)\big)^2+H_z(t)\big(1-H_z(t)\big)$. By Tonelli,
\[
\mathbb{E}_{y\sim h_z}[\mathrm{twCRPS}(P_z,y)]
=\int_{\mathcal{T}} w_z\,(P_z-H_z)^2\,dt
+ R_z,
\]
with $R_z=\int_{\mathcal{T}} w_z\,H_z(1-H_z)\,dt$
independent of $P_z$. Hence the divergence is $d_z(P_z\|h_z)=\int_{\mathcal{T}}w_z\,(P_z-H_z)^2\,dt\ge0$, with equality if and only if $P_z(t)=H_z(t)$ for Lebesgue-almost every $t$ in $\{w_z>0\}$. Since $w_z>0$ on the interior of $\mathcal{T}$ and both $P_z,H_z$ are right-continuous, $P_z=H_z$ everywhere on $\mathcal{T}$; this pins down $\lim_{t\uparrow T_h}P_z(t)=\lim_{t\uparrow T_h}H_z(t)=1-\beta_z$, so the no-response atom masses coincide and $P_z=h_z$ as mixed laws. Therefore $\mathbb{E}_{y\sim h_z}[\mathrm{twCRPS}]$ is uniquely minimized at $h_z$, i.e.\ the score is strictly proper, and $1-d_z/W_z=1$ exactly when $P_z=h_z$. A deterministic forecast has a step CDF and a law on the post-completion half-line has $P_z\equiv0$ on $(-\infty,0)$; either differs from a dispersed or anticipatory $h_z$ on a set of positive $w_z$-measure, forcing $d_z>0$. Summing over $z$ with $\pi(z)>0$ yields $S_{\mathrm{T}}\in[0,1]$ with $S_{\mathrm{T}}=1$ iff $P_z=h_z$ for all such $z$.
\end{proof}

\begin{corollary}[Composite strict propriety]
\label{cor:composite}
The length aggregate $S_{\mathrm{L}}$ of \eqref{eq:length} is a CRPS divergence and the overlap score is a bounded penalty, both in $[0,1]$; any fixed convex combination of $S_{\mathrm{T}}$, $S_{\mathrm{L}}$, and $S_{\mathrm{O}}$, and the unweighted mean of the four category scores, is a nonnegative weighted sum of strictly proper divergence terms, hence lies in $[0,1]$ and is maximized at the value $1$ exactly when the model reproduces every human conditional timing, duration, and overlap law it combines.
\end{corollary}
\begin{proof}
$S_{\mathrm{L}}$ is the unweighted twCRPS of Proposition~\ref{prop:bound} (weight $w\equiv1$) applied to durations, so it is strictly proper and in $[0,1]$ with unique maximizer $Q_z=g_z$; $s_{\mathrm{O}}=\exp(-\kappa\sum(1-\psi)d)\in(0,1]$ equals $1$ iff there is no low-completion competitive overlap. A convex combination $\sum_c\lambda_c S_c$ with $\lambda_c\ge0$, $\sum_c\lambda_c=1$ of terms each in $[0,1]$ lies in $[0,1]$, and since each $S_c\le1$ with equality only at its own optimum, the combination attains $1$ iff every active term does; the four-category mean is the special case of equal weights.
\end{proof}

\begin{proposition}[Reduction to binary fixed-window scoring]
\label{prop:reduction}
Let the posterior be degenerate at a single intent $z_0$ and the model emit the deterministic forecast $P_{z_0}=\mathbf{1}[t\ge y_e]$. Replacing the smooth weight $w_{z_0}$ by the indicator weight $w(t)=\delta\text{-limit at }W$, i.e.\ taking $w_\eta(t)=\eta^{-1}\mathbf{1}[W<t\le W+\eta]$ and letting $\eta\to0^+$, the normalized twCRPS score $1-\mathrm{twCRPS}/W_z$ converges to the response-within-window indicator $\mathbf{1}[\,0<y_e\le W\,]$ whose episode mean is exactly the takeover-style statistic of Full-Duplex-Bench \cite{lin2025full}; $W\to\infty$ recovers the unwindowed takeover rate, and the two-sided indicator weight on $\{-W'<t\le W\}$ recovers the windowed overlap statistics of \cite{lin2026full}.
\end{proposition}
\begin{proof}
Concentrating $w_\eta$ at the single threshold $t=W$ and normalizing, $1-\mathrm{twCRPS}/W_z$ tends to $1$ when the realized onset precedes $W$ (a response within the window) and to $0$ otherwise, that is to $\mathbf{1}[0<y_e\le W]$; abstention gives $P_{z_0}\equiv0$ and score $0$. The episode mean is the windowed takeover rate, $W\to\infty$ removes the upper edge, and the symmetric two-sided indicator weight recovers the overlap window, so the binary fixed-window metric is the degenerate indicator-weight limit of the twCRPS score.
\end{proof}

\begin{proposition}[Identifiability of the intent mixture]
\label{prop:ident}
If the per-intent laws $h_z$ are pairwise distinct ex-Gaussian/shifted-log-normal laws with an atom at $\varnothing$, intents are drawn i.i.d.\ from a speaker prior $\pi_s$, and the response law given $z$ does not otherwise depend on $s$, then $\pi_s$ and $\{h_z\}$ are identifiable up to label permutation from the per-speaker marginal response law.
\end{proposition}
\begin{proof}
Atom masses separate from continuous parts by mutual singularity, so it suffices that the continuous components be linearly independent. The ex-Gaussian characteristic function $\exp(i\mu\omega-\sigma^2\omega^2/2)(1-i\tau\omega)^{-1}$ has a unique simple pole at $\omega=-i/\tau$; ordering by $\tau$ and taking residues rules out nontrivial vanishing combinations, Gaussian-mixture identifiability handles equal-$\tau$ groups, and shifted log-normals are separated by their support edge. Linear independence makes $\pi_s$ unique, licensing the memory profiles: $\pi_s$ is recoverable from that speaker's episodes and nothing less.
\end{proof}

\subsection{Reconciling the Population Optimum with the Empirical Topline}
\label{sec:reconcile}
Proposition~\ref{prop:bound} fixes the population optimum at $S_{\mathrm{T}}=1$, attained only by $h_z$, yet the empirical human topline in Table~\ref{tab:main} is $0.86$ because the per-episode estimator scores a single held-out continuation $y_e$ against the reference. Two finite-sample effects separate it from $1$: each episode offers one realized continuation, so even a forecaster equal to $h_z$ pays the irreducible single-draw term $R_z/W_z>0$ from the proof above; and a held-out speaker's timing deviates from any pooled $h_z$, a gap Proposition~\ref{prop:ident} shows shrinks only with that speaker's own evidence. Together they account for the $0.14$ shortfall, so the topline is a principled estimate of the population optimum.

\begin{table*}[t]
\caption{Main results on TACT by Full-Duplex-Bench scenario category (Pause: pause handling, Backch.: backchanneling, Smooth: smooth turn-taking, Interr.: user interruption); the composite is the category mean with 95\% BCa bootstrap intervals. These are scenario types, distinct from the six latent intents of Fig.~\ref{fig:heatmap}; in particular the Backch.\ scenario is not the \textsc{bck} intent. Ranks compare TACT and reproduced binary v1-style orderings (Spearman 0.55).}
\label{tab:main}
\centering
\scriptsize
\renewcommand{\arraystretch}{0.84}
\setlength{\tabcolsep}{4.2pt}
\begin{tabular}{lcccccc}
\toprule
System & Pause & Backch. & Smooth & Interr. & Composite [95\% CI] & Rank (TACT / v1) \\
\midrule
dGSLM \cite{nguyen2023generative} & 0.21 & 0.28 & 0.27 & 0.20 & 0.24 [0.22, 0.26] & 11 / 11 \\
Moshi \cite{defossez2024moshi} & 0.22 & 0.25 & 0.41 & 0.31 & 0.30 [0.28, 0.32] & 10 / 2 \\
SyncLLM \cite{veluri2024beyond} & 0.26 & 0.33 & 0.37 & 0.28 & 0.31 [0.29, 0.33] & 9 / 10 \\
Freeze-Omni \cite{wang2025freeze} & 0.35 & 0.27 & 0.36 & 0.34 & 0.33 [0.31, 0.35] & 8 / 9 \\
Qwen2.5-Omni \cite{xu2025qwen} & 0.36 & 0.29 & 0.40 & 0.35 & 0.35 [0.33, 0.37] & 7 / 8 \\
MiniCPM-o 2.6 \cite{openbmb2025minicpmo} & 0.37 & 0.30 & 0.42 & 0.36 & 0.36 [0.34, 0.38] & 6 / 6 \\
PersonaPlex \cite{roy2026personaplex} & 0.40 & 0.31 & 0.45 & 0.36 & 0.38 [0.36, 0.40] & 5 / 5 \\
GPT-4o Realtime \cite{openai2024gpt4o} & 0.44 & 0.35 & 0.49 & 0.43 & 0.43 [0.41, 0.45] & 4 / 4 \\
gemini-3.1-flash-live-minimal \cite{google2026geminilive} & 0.45 & 0.33 & 0.52 & 0.45 & 0.44 [0.42, 0.46] & 3 / 1 \\
gemini-3.1-flash-live-high \cite{google2026geminilive} & 0.52 & 0.36 & 0.48 & 0.47 & 0.46 [0.44, 0.48] & 2 / 7 \\
gpt-realtime-2 \cite{openai2026gptrealtime2} & 0.50 & 0.38 & 0.53 & 0.46 & 0.47 [0.45, 0.49] & 1 / 3 \\
\midrule
VAP-gated cascade \cite{ekstedt2022voice,inoue2024realtime} & 0.44 & 0.20 & 0.47 & 0.28 & 0.35 [0.33, 0.37] & -- \\
Never respond & 0.42 & 0.10 & 0.06 & 0.14 & 0.18 [0.17, 0.19] & -- \\
Human topline & 0.84 & 0.87 & 0.88 & 0.85 & 0.86 [0.84, 0.88] & -- \\
\bottomrule
\end{tabular}
\end{table*}

\section{Experimental Setup}
\label{sec:setup}

We evaluate eleven systems: open models dGSLM \cite{nguyen2023generative}, Moshi \cite{defossez2024moshi}, SyncLLM \cite{veluri2024beyond}, Freeze-Omni \cite{wang2025freeze}, Qwen2.5-Omni \cite{xu2025qwen}, MiniCPM-o~2.6 \cite{openbmb2025minicpmo}, PersonaPlex \cite{roy2026personaplex}, and proprietary points GPT-4o Realtime \cite{openai2024gpt4o}, gpt-realtime-2 \cite{openai2026gptrealtime2}, Gemini 3.1 Flash Live \cite{google2026geminilive} at minimal and high thinking. Open models run locally in streaming mode, proprietary systems via realtime APIs \cite{lin2026full}; prompt models receive synopses, dGSLM priming audio. Reference baselines include a VAP-gated cascade \cite{ekstedt2022voice,inoue2024realtime}, a never-respond policy, and the human topline. Reproduced binary v1 metrics \cite{lin2025full} match published numbers (Moshi pause takeover $>0.98$, median latency $\approx0.26$~s). The validity study uses 1,280 held-out episodes judged by five raters in 88 system-condition cells.

\section{Results}
\label{sec:results}

\subsection{Main Comparison}
Table~\ref{tab:main} reports main results. The gap to humans is large: the best system, gpt-realtime-2, attains 0.47 vs 0.86 human topline, while binary v1 metrics make systems look adequate. Rankings reorder substantially (correlation 0.55): Moshi, second under binary scoring due to 0.26~s latency, falls to tenth because instant takeover fails on \textsc{ref}, \textsc{rhe}, and \textsc{hold}. Gemini points expose a trade-off invisible to binary scoring: minimal responds fastest (0.31~s) and tops v1, while high waits median 0.74~s, falling to seventh under v1 yet overtaking it on TACT. The VAP cascade matches mid-tier models (0.35) and beats open models on timing, but degrades on backchannels.

\subsection{The Over-Eagerness Pathology}
The scenario categories of Table~\ref{tab:main} and the intents of Fig.~\ref{fig:heatmap} are orthogonal cuts: each scenario composite is the prior-weighted average of the Fig.~\ref{fig:heatmap} cells in its rows. Over-eagerness is traceable through \textsc{ref}, \textsc{rhe}, and \textsc{hold}. All eleven systems do best on \textsc{ans} (0.38--0.78) and collapse on \textsc{ref} (0.09--0.36) and \textsc{rhe} (0.06--0.30). Humans respond within 1~s in 78\% of \textsc{ans} episodes but only 12\% of \textsc{ref} and 7\% of \textsc{rhe}; Moshi does so in 95--97\% and conservative models in 62--76\%, withholding the floor on \textsc{rhe} in $<4\%$ (open) to 24\% (best closed) vs 74\% for humans. Rigidity is two-sided: humans launch 63\% of invited backchannels in overlap, vs 14\% for PersonaPlex and $<6\%$ elsewhere. Models fail to be slow when reflection is warranted or early when invited, recurring as length miscalibration (backchannels receive median 8.4~s turns vs 1.1~s human).

\begin{figure}[t]
\centering
\includegraphics[width=0.88\columnwidth]{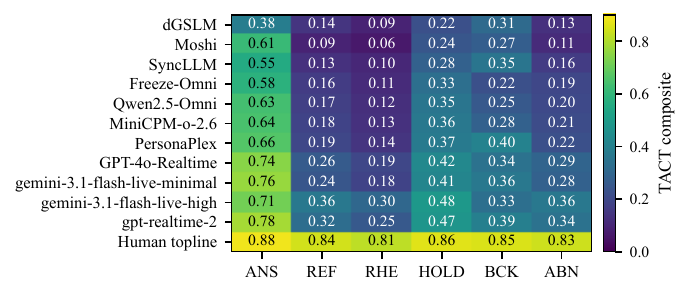}
\caption{Each cell is the TACT composite $S$ of \eqref{eq:stime} restricted to episodes whose posterior mode is the column intent (the intent-conditioned composite; intent-view, orthogonal to the scenario-category axis of Table~\ref{tab:main}). Systems approach humans only on \textsc{ans} and degrade two- to ten-fold on \textsc{ref}, \textsc{rhe}, and \textsc{abn}; the overlap-optimal \textsc{bck} column shows the anticipatory side of the rigidity.}
\label{fig:heatmap}
\end{figure}

\subsection{Memory Ablation: No Per-Speaker Adaptation}
Re-running each system with profiles swapped across speakers of opposite behavioral type, we measure $\mathrm{JSD}_2$ between onset distributions under true and swapped profiles. Prompt-capable models receive the prompt synopsis, Moshi receives profile features through its synchronized text stream, and dGSLM receives matched priming audio. Held-out human behavior shifts substantially ($\mathrm{JSD}_2=0.23$) while every system is essentially invariant ($\Delta S \le 0.01$ and $\mathrm{JSD}_2 \le 0.024$ across all eleven systems, with 0.024 for gpt-realtime-2 and 0.003 for dGSLM being the swap extrema), demonstrating that current models settle on a speaker-independent policy where human behavior is speaker-conditioned \cite{sap2022neural,kim2023fantom}.

\subsection{Metric Validity}
Over 88 cells judged by held-out humans, TACT reaches Spearman $\rho=0.81$ ($[0.77,0.84]$), vs $0.46$ for binary composite and $0.39$ for latency. An intent-conditioned binary baseline attains only $\rho=0.62$, showing that continuous twCRPS is essential: fixed windows cannot distinguish timing within windows or represent anticipatory, delayed, and absent onsets jointly. Ablating memory, calibration, or prosody yields $\rho=0.71$, $0.78$, and $0.74$. Calibration reduces ECE from 8.4\% to 2.1\% ($\lambda=0.3$).

\subsection{Fairness and Robustness Slices}
Composites by slice show the human topline flat across gender, age, first-language region, SNR, and articulation rate (0.84--0.87). Systems are not: the best loses 0.07 between North American and Indian-English speakers and 0.06 between high- and mid-SNR conditions (concentrated in timing), with smaller gender (0.02) and age (0.05) gaps.

\subsection{Limitations}
The six-class taxonomy discretizes a continuum and $\alpha=0.73$ leaves label noise; kernels are fit on English corpora, though the released recipe enables refitting where norms differ \cite{stivers2009universals,inoue2024multilingual}; the judge in \eqref{eq:post} may share failure modes (an annotator-only mode ships); and anticipatory scoring leans on VAP completion estimates $\psi$ that released marks would replace.

\section{Conclusion}
\label{sec:conclusion}

TACT replaces binary fixed-window turn-taking with a strictly proper, intent-conditioned twCRPS over human floor-transfer-offset laws, per-speaker memory, and a calibrated posterior; data, kernels, and prompts are available in the Supplementary Material. When to speak is an inference over a speaker's latent intent, not a latency problem.

\section*{Acknowledgment}
The authors disclose that Claude Opus 4.8 was used for editing and rewriting all sections (Abstract, Introduction, Related Work, Methods, Experiments, Results, Conclusion) to improve the flow and presentation, and for assisting in implementation of the code.

\clearpage
\bibliographystyle{IEEEtran}
\bibliography{refs}

\end{document}